\documentclass{article}

\usepackage[preprint]{neurips_2025}
\usepackage{float} 

\usepackage{wrapfig} 
\usepackage{amsmath}
\usepackage{hyperref}
\usepackage{svg}
\usepackage{amsmath}
\usepackage{amssymb}
\usepackage{mathtools}
\usepackage{amsthm}
\usepackage{booktabs}
\usepackage{multirow}
\usepackage{graphicx}  
\usepackage{multirow}  

\usepackage{graphicx}  
\usepackage{multirow}  
\theoremstyle{plain}
\newtheorem{theorem}{Theorem}[section]
\newtheorem{proposition}[theorem]{Proposition}
\newtheorem{lemma}[theorem]{Lemma}

\theoremstyle{definition}

\theoremstyle{remark}

\usepackage{algorithm}
\usepackage{algpseudocode}

\usepackage[utf8]{inputenc} 
\usepackage[T1]{fontenc}    
\usepackage{hyperref}       
\usepackage{url}            
\usepackage{booktabs}       
\usepackage{amsfonts}       
\usepackage{nicefrac}       
\usepackage{microtype}      
\usepackage{xcolor}         

\title{SplitLoRA: Balancing Stability and Plasticity 
in Continual Learning 
Through Gradient Space Splitting}

\author{
Haomiao Qiu\textsuperscript{1,2}\hspace{0.5em}, 
Miao Zhang\textsuperscript{1}\thanks{Co-corresponding Authors.}\hspace{0.5em}, 
Ziyue Qiao\textsuperscript{2}\footnotemark[1]\hspace{0.5em}, 
Weili Guan\textsuperscript{1}, 
\textbf{Min Zhang\textsuperscript{1}, Liqiang Nie\textsuperscript{1}} \\
\textsuperscript{1} Harbin Institute of Technology (Shenzhen) \\
\textsuperscript{2} Great Bay University \\
\texttt{24B951058@stu.hit.edu.cn, zhangmiao@hit.edu.cn, zyqiao@gbu.edu.cn} \\
\texttt{honeyguan@gmail.com, zhangmin2021@hit.edu.cn, nieliqiang@gmail.com} \\
}

\begin{document}

\maketitle

\begin{abstract}
Continual Learning (CL) requires a model to learn multiple tasks in sequence while maintaining both stability—preserving knowledge from previously learned tasks, and plasticity—effectively learning new tasks. Gradient projection has emerged as an effective and popular paradigm in CL, where it partitions the gradient space of previously learned tasks into two orthogonal subspaces: a primary subspace and a minor subspace. New tasks are learned effectively within the minor subspace, thereby reducing interference with previously acquired knowledge. However, existing Gradient Projection methods struggle to achieve an optimal balance between plasticity and stability, as it is hard to appropriately partition the gradient space. In this work, we consider a continual learning paradigm based on Low-Rank Adaptation (LoRA), which has gained considerable attention due to its efficiency and wide applicability, and propose a novel approach for continual learning, called SplitLoRA. We first provide a theoretical analysis of how subspace partitioning affects model stability and plasticity. Informed by this analysis, we then introduce an effective method that derives the optimal partition of the gradient space for previously learned tasks. This approach effectively balances stability and plasticity in continual learning. Experimental results on multiple datasets demonstrate that the proposed method achieves state-of-the-art performance. 
\end{abstract}

\section{Introduction}
Continual Learning (CL) refers to a model’s ability to sequentially learn new tasks while retaining knowledge from previously learned tasks~\cite{parisi2019continual}. This contrasts with traditional machine learning paradigms, which assume that models are trained on a fixed dataset where all data is available at once. In the CL setting, the challenge lies in maintaining performance on previous tasks while adapting to new ones, necessitating a balance between stability and plasticity. In recent years, orthogonal projection methods have demonstrated strong performance in continual learning tasks. These methods require storing the subspace spanned by the gradients of previous tasks in memory. During new task training, the gradient of the current task is projected onto the minor subspace of the previous task's gradient subspace, reducing the interference of new task updates with previously learned knowledge.

Parameter-Efficient Fine-Tuning (PEFT)~\cite{hu2021lora,houlsby2019parameter,DBLP:conf/eccv/JiaTCCBHL22} enables efficient fine-tuning for new tasks by keeping the pre-trained model parameters unchanged while introducing a small subset of trainable parameters. Due to its advantages in computational efficiency and performance, PEFT methods have gained increasing popularity in continual learning~\cite{wang2022learning,smith2023coda,gao2023unified,liang2024inflora,lu2024visual}. Combining orthogonal projection with PEFT can better leverage the knowledge of the pre-trained model, allowing for faster adaptation to new tasks. 

However, existing methods~\cite{liang2024inflora,lu2024visual} typically determine the subspace dimension for each module within the model by using a predefined threshold based on the cumulative sum of squared singular values. This approach enforces a uniform partitioning rule across all modules, ignoring the fact that different modules contribute unequally to knowledge retention~\cite{jiang2024taia, dai2021knowledge, geva2020transformer}. As a result, it fails to achieve an optimal trade-off between stability and plasticity.

In this paper, we theoretically analyze the relationship between the size of the minor gradient subspace of previous tasks and the upper bound of loss increments across all tasks. Furthermore, we model its impact on both stability and plasticity.
In practice, we build upon the LoRA framework and propose a novel method called SplitLoRA for CL tasks.
Specifically, to minimize the upper bound of the total task loss growth, we construct an optimization problem to determine the optimal size of minor subspace and derive an approximate solution to balance stability and plasticity. 

Our contributions are summarized as follows:
\begin{itemize}
    \item We theoretically model the impact of the gradient subspace size of previous tasks on stability and plasticity in orthogonal projection based continual learning in Theorem~\ref{thm:upper_bound_forandcur} and derive an approximate optimal minor subspace in CL.
    \item We introduce SplitLoRA, a novel PEFT framework. By projecting the minor subspace onto the LoRA dimension reduction matrix \( \mathbf{A}_t \) via a random projection and optimizing only \( \mathbf{B}_t \), SplitLoRA ensures that updates remain confined to the minor subspace, thereby achieving an effective balance between stability and plasticity.
    \item  Our method achieves state-of-the-art performance across multiple datasets, surpassing existing CL methods by 2\%–5\% on different datasets.  
\end{itemize}
\section{Related Work}  
\subsection{Parameter-Efficient Fine-Tuning}  
Parameter-efficient fine-tuning modifies pre-trained models by introducing a small set of trainable parameters while keeping the original model frozen, significantly reducing computational costs while maintaining strong performance. Adapter~\cite{houlsby2019parameter} fine-tunes small modules added to multiple layers, while Prompt-tuning~\cite{DBLP:conf/emnlp/LesterAC21} and Prefix-tuning~\cite{DBLP:conf/acl/LiL20} inject trainable tokens into Transformer layers. LoRA~\cite{hu2021lora} decomposes weight updates into low-rank matrices, tuning only these structures. Despite training fewer parameters, PEFT methods often achieve comparable or superior performance~\cite{zaken2022bitfit,fu2022adapterbias,hu2021lora,DBLP:conf/nips/MahabadiHR21}. Initially developed for NLP, PEFT has been extended to vision tasks, with methods such as Visual Prompt Tuning (VPT)~\cite{DBLP:conf/eccv/JiaTCCBHL22} and AdapterFormer~\cite{chen2022adaptformer} achieving performance on par with full fine-tuning.  

\subsection{Continual Learning}  
Continual learning methods fall into three main categories: regularization-based, memory-based, and expansion-based. Regularization-based approaches~\cite{zenke2017continual,jung2020continual,aljundi2018memory,kirkpatrick2017overcoming} constrain significant changes to key parameters to mitigate catastrophic forgetting. Memory-based methods~\cite{DBLP:conf/nips/AljundiBTCCLP19,DBLP:conf/nips/AljundiLGB19,sun2022exploring,DBLP:conf/nips/LiangL23} retain prior task information in a buffer, allowing models to revisit past knowledge. Expansion-based techniques~\cite{rusu2016progressive,DBLP:conf/nips/Hung0WCCC19,li2019learn} dynamically expand the model architecture to accommodate new tasks while preserving learned representations.  

\textbf{Gradient Projection in CL.}  
Gradient projection~\cite{zeng2019continual,farajtabar2020orthogonal,saha2021gradient} mitigates task interference by constraining updates to directions orthogonal to previous tasks. Orthogonal Weight Modulation~\cite{zeng2019continual} learns a projector matrix to prevent new gradients from overwriting prior knowledge. Orthogonal Gradient Descent~\cite{farajtabar2020orthogonal} projects new gradients onto the orthogonal complement of previous task gradients. Gradient Projection Memory (GPM)~\cite{saha2021gradient} stores subspace bases of old task data and projects new gradients onto their orthogonal complement. Trust Region Gradient~\cite{lin2022trgp} enhances forward knowledge transfer by leveraging task-related representations.  

\textbf{PEFT in CL.}  
With the rise of pre-trained models~\cite{he2022masked,dosovitskiy2020image,DBLP:conf/naacl/DevlinCLT19}, continual learning has shifted toward leveraging them rather than training from scratch. While some approaches~\cite{boschini2022transfer,zheng2023preventing} fine-tune pre-trained models fully, this is often inefficient. To address this, PEFT methods have been explored in continual learning, with studies~\cite{smith2023coda,wang2022learning,khan2023introducing,yu2024personalized} integrating prompt-tuning to improve class-incremental learning. A unified framework~\cite{gao2023unified} further combines various PEFT techniques, including prompt-tuning, LoRA, and Adapter, into continual learning.

\section{Preliminary}
\subsection{Continual Learning Formulation}
In CL, there are $T$ tasks $\mathcal{T}_1,...,\mathcal{T}_T$, each task includes data: $\mathcal{D}_t=\{(\mathbf{x}_i^t,y_i^t)\}_{i=1}^{n_t}$, 
where $\mathbf{x}_i^t\in \mathbb{R}^d$ is the input and $y_i^t\in \mathbb{R}$ is the label. 
The goal of CL is to achieve an overall optimal performance across all tasks. Let $\mathbf{W}_T$ represent the model parameters after training on the last task $T$. The loss on task $t$, denoted as $\mathcal{L}_{t}(\mathbf{W}_T)$, measures the performance of $\mathbf{W}_T$ on task $t$. Let $\mathbf{W}_T^*$ represent the optimal parameters. Then, the objective is to minimize the total loss across all tasks:
\begin{align}
    \mathbf{W}_T^* = \mathop{\text{argmin}}_{\mathbf{W}_T} \mathcal{L}_{all}(\mathbf{W}_T) = \mathop{\text{argmin}}_{\mathbf{W}_T} \sum_{t=1}^{T} \mathcal{L}_t(\mathbf{W}_T).
\end{align}

\subsection{LoRA-based Continual Learning}
LoRA~\cite{hu2021lora} is a low-rank based PEFT method that reduces the number of parameters by decomposing the weight matrix into the product of two low-rank matrices. Specifically, for a linear layer, 
the extra weight matrix $\Delta \mathbf{W}$ is decomposed into two low-rank matrices $\mathbf{A}$ and $\mathbf{B}$ as :
$\Delta  \mathbf{W} =  \mathbf{A} \mathbf{B}$,
where $\mathbf{A}\in \mathbb{R}^{d_1 \times r}$ and $\mathbf{B}\in \mathbb{R}^{r\times d_2}$,
and $r$ is the dimension of the low-rank. In this way, the number of parameters of the weight matrix is reduced from $d_1d_2$ to $2dr$. 
During training, we learn $\mathbf{A}$ and $\mathbf{B}$ by minimizing the loss function of the current task. During testing, we recover $\mathbf{W} =  \mathbf{W}_0 + \mathbf{A} \mathbf{B}$ and use it for forward propagation.

CL generally initializes an additional LoRA for the new task $t$, while the LoRA of old tasks also participates in the forward process~\cite{liang2024inflora,wang2023olora}.
In the current task $t$, the forward process of the linear layer is
\begin{align}
    \mathbf{Y} =\mathbf{W}_t\mathbf{X}= (\mathbf{W}_{t-1} + \mathbf{A}_t  \mathbf{B}_t )\mathbf{X},
\end{align}
and only $\mathbf{A}_t$ and $\mathbf{B}_t$ are trained during training.

\begin{figure*}[t]
\setlength{\belowcaptionskip}{-8pt} 
\begin{center}
\centerline{\includegraphics[width=\textwidth]{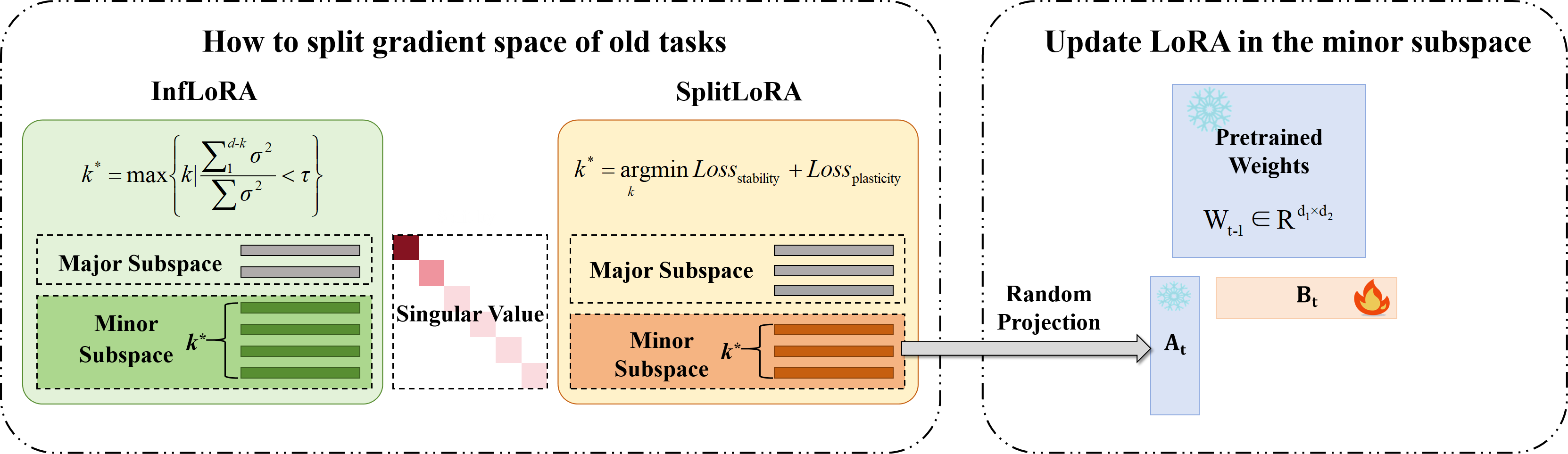}}
\caption{An overview of our proposed SplitLoRA. During the learning of the $t$-th task, the gradient space of tasks 1 to $t-1$ is decomposed into major and minor subspaces. InfLoRA determines $k^{*}$ solely based on a predefined threshold, whereas SplitLoRA balances stability loss and plasticity loss to determine $k^{*}$. Then the minor subspaces are randomly projected onto the low-dimensional matrix $A$ of LoRA and fixed, while only $B$ is trained. Specifically, 
$W_{t-1} = W_0 + \Sigma_{i=1}^{t-1} A_i B_i$,
where $W_0$ represents the pre-trained model weights, and $k$ denotes the size of the minor subspace.}
\label{pipeline}
\end{center}
\end{figure*}

\section{SplitLoRA}
In this section, we introduce SplitLoRA, a PEFT method for CL that mitigates catastrophic forgetting by partitioning the gradient space.
Unlike existing Gradient Projection-based methods~\cite{liang2024inflora,lu2024visual,jin2021gradient}, which typically define the minor subspace solely based on the sum of squared singular values being below a predefined threshold, SplitLoRA determines the optimal subspace size by analyzing the impact of the minor subspace on stability loss and plasticity loss.
We first introduce gradient projection, then model the effect of subspace partitioning on learning dynamics and formulate an optimization problem to derive the approximate optimal size of the minor subspace. Finally, we present how to construct the low-rank projection matrix within this subspace to enhance CL.
The entire process of SplitLoRA is illustrated in Figure ~\ref{pipeline}.

\subsection{Orthogonal Decomposition based Gradient Projection}
Empirically, training on new tasks often leads to performance degradation on previously learned tasks due to interference between the gradients of new and old tasks, resulting in stability loss. To address this issue, GPM~\cite{saha2021gradient} orthogonally decomposes the gradient space of previous tasks into a major subspace and a minor subspace, and constrains the update direction of the new task within the minor subspace. Our work is also built upon orthogonal decomposition.

In CL, we aim to maintain an average gradient space $\mathbf{G}^\text{old}$ for all previous tasks.
Specifically, after training the task $t-1$, we re-feed the data from this task into the model and calculate the average gradient $\mathbf{G}^\text{new}_{t-1}$ of $W$ throughout this process. Finally, we compute the average gradient space $\mathbf{G}^\text{old}_{t}$ for the previous $t-1$ tasks:
\begin{align}
\label{eq:old}
\mathbf{G}^\text{old}_{t} = \frac{1}{t-1}((t-2)\mathbf{G}^\text{old}_{t-1}+\mathbf{G}^\text{new}_{t-1}).
\end{align}
For the first task, $\mathbf{G}^\text{old}_{1}$ is equal to zero.
Next, we receive the data from task $t$. We perform a Singular Value Decomposition (SVD) on the gradient $\mathbf{G}^\text{old}_{t}$.
Larger singular values correspond to singular vectors that dominate in describing the vector's importance.
We select the last $k$ left singular vectors of $\mathbf{\hat U}_t$ as the minor subspace: 
\begin{align}
\mathbf{\hat U}_{t}, \mathbf{\hat \Sigma}_{t} ,\mathbf{\hat V}_{t}^\top = \text{SVD}(\mathbf{G}^\text{old}_{t}), \quad \quad  \quad  \mathbf{\hat{U}}_{t}^{k} &= \mathbf{\hat{U}}_{t}\bigl[:,\, -k:\bigr].
\end{align}
The gradient of previous tasks has a much smaller component in the minor subspace compared to the major subspace. Therefore, projecting the gradient of the new task onto the minor subspace will result in minimal interference. 

A common projection method is to construct a projection transformation matrix. For simplicity, we consider a linear layer $\mathbf{W}$ in a model.
As the model update $\Delta \mathbf{W}$ is determined by the gradient, projecting the model update onto the minor subspace is equivalent to constraining the gradient direction, which helps mitigate interference. 
Specifically, we project the model update $\Delta \mathbf{{W}}$ onto the minor subspace. The projection result is given by:
\begin{align}
\Delta \mathbf{\hat{W}}=\text{proj}_{\text{col}(\mathbf{\hat U}^k_{t})}(\Delta \mathbf{{W}})=\mathbf{\hat U}^k_t{\mathbf{\hat U}^{k\top}_{t}}\Delta \mathbf{{W}},
\end{align}
where $\mathbf{\hat U}^k_{t}$ is the projection subspace. Since the gradients of previous tasks are distributed primarily in the major subspace, therefore the projection ensures that the updates primarily benefit the new task while minimally affecting the performance of old tasks.

\subsection{Minor Space Setting for Old Tasks}
\label{sec:how}
Projecting the gradient onto the minor subspace can mitigate interference with previous tasks.
And the larger the minor subspace size $k$, the larger the learning space for the new task, leading to better plasticity. However, as the gradient components of previous tasks in the minor subspace increase, stability deteriorates.

Previous methods~\cite{jin2021gradient,lin2022trgp,liang2024inflora,lu2024visual} compute the sum of the squared singular values corresponding to the minor subspace, ensuring that it remains below a predefined threshold $\tau$. Among all values of $k$ that satisfy this condition, they select the largest one:
\begin{align}
k^*=\text{max} \left \{ k|\frac{\sum_{i=1}^{d-k} \sigma_i^2}{\sum \sigma^2} < \tau \right \}.
\end{align}
The size of the minor subspace, denoted as $k$, is a crucial parameter that affects both model stability and plasticity. However, previous methods determine $k$ based on a predefined threshold $\tau$, which is merely a hyperparameter and does not effectively balance stability and plasticity. Thus, we proceed to analyze how subspace selection impacts the loss across all tasks.
Based on the smoothness of the loss function $\mathcal{L}$, we can derive an upper bound on the loss incurred due to parameter updates in CL.

\begin{proposition}[Upper Bound on Loss Increase]
\label{thm:upper_bound}
Consider a model with a linear layer updated from \( \mathbf{W}_{t-1} \) to \( \mathbf{W}_t = \mathbf{W}_{t-1} + \Delta \mathbf{W}_t \).  
Assume the loss function is $L$-smooth and that the first $t{-}1$ tasks were trained with updates constrained to be orthogonal to the gradients of previous tasks. Then, the total loss change over tasks \( 1, \dots, t \) is bounded by:
\begin{align}
\label{eq:upper}
    \sum_{i=1}^{t} \left(\mathcal{L}_{i}(\mathbf{W}_{t}) - \mathcal{L}_{i}(\mathbf{W}_{t-1})\right)
    \leq  \underbrace{-(t{-}1) \left \langle \Delta \mathbf{W}_t, \mathbf{G}^{\text{old}}_t \right \rangle}_{\text{Stability Loss}}
    \underbrace{- \left \langle \Delta \mathbf{W}_t, \mathbf{G}_t \right \rangle}_{\text{Plasticity Loss}} 
    + \frac{(t{-}1)L}{2} \| \Delta \mathbf{W}_t \|_F^2,
\end{align}
where \( \mathbf{G}_t = \nabla \mathcal{L}_t(\mathbf{W}_t) \) is the gradient for task $t$, and \( \mathbf{G}^{\text{old}}_t = \frac{1}{t{-}1} \sum_{i=1}^{t-1} \mathbf{G}_i \) is the average gradient of previous tasks. \( \langle \cdot, \cdot \rangle \) denotes the Frobenius inner product.
\end{proposition}

This result shows that parameter updates affect both the current and past tasks. The term \( \langle \Delta \mathbf{W}_t, \mathbf{G}^{\text{old}}_t \rangle \) captures interference with past tasks (stability loss), while \( \langle \Delta \mathbf{W}_t, \mathbf{G}_t \rangle \) reflects progress on the current task (plasticity gain). The squared norm \( \| \Delta \mathbf{W}_t \|_F^2 \) acts as a regularization term controlled by the smoothness constant $L$.

Next, we discuss how to choose the minor subspace in gradient projection to minimize the combined stability and plasticity losses.
Building on Proposition 1, we can theoretically analyze how stability loss and plasticity loss vary as a function of $k$.
From Eq.~\eqref{eq:upper}, after replacing $\Delta \mathbf{W}_{t}$ with $\Delta \mathbf{\hat W}_{t}$, where $\Delta \mathbf{\hat W}_{t} = \mathbf{U}_t^k \mathbf{U}_t^{k\top} \Delta \mathbf{W}_{t}$ is the projected update onto the minor subspace,
we can express the stability loss $\mathcal{L}_{t}^{S}(\mathbf{W}_{t})$ 
and the plasticity loss $\mathcal{L}_{t}^{P}(\mathbf{W}_{t})$ as follows:
\begin{align}
\mathcal{L}_{t}^{S}(\mathbf{W}_{t}) &= - (t-1) \left <\Delta \mathbf{\hat W}_{t}, \mathbf{G}^{\text{old}}_{t} \right >, \\
\mathcal{L}^{P}_{t}(\mathbf{W}_{t}) &= - \left <\Delta \mathbf{\hat W}_{t}, \mathbf{G}_{t} \right >.
\end{align}
The stability  loss is proportional to the alignment between the projected update $\Delta \mathbf{\hat W}_{t}$ 
and the gradient of old tasks $\mathbf{G}^\text{old}_{t}$, 
while the plasticity loss depends on the alignment of $\Delta \mathbf{\hat W}_{t}$ 
with the gradient of the new task $\mathbf{G}_{t}$.
Then  we define the error function $\epsilon(k)$, which quantifies the proportion of these minor directions and is given by:
\begin{align}
\epsilon(k) = \frac{\sum_{i=d-k+1}^d \sigma_i}{\sum_{i=1}^d \sigma_i}.
\end{align}
$\epsilon(k)$ measures the interference error caused by updating the model within the minor subspace on old tasks.
Based on Proposition ~\ref{thm:upper_bound}, we can derive the following theorem:
\begin{theorem}
\label{thm:upper_bound_forandcur}
Let \( \mathbf{W}_{t-1} \) denote the weight matrix of a linear layer in the model, updated as  $\mathbf{W}_t = \mathbf{W}_{t-1} + \Delta \mathbf{\hat W}_{t}= \mathbf{W}_{t-1} +  \mathbf{U^k}_t \mathbf{U^k}_t^{\top} \Delta \mathbf{W}_{t}$ . Since the update direction of the new task is unknown, we assume that it is uniformly distributed across all directions. that is to say, 
 \( \Delta \mathbf{W}_t \) has the same expected projection value across different feature directions of \( \mathbf{G}_t \), we provide the expected values of the stability loss :
\begin{align}
\mathbb{E}[\mathcal{L}_{t}^{S}(\mathbf{W}_{t})] = -(t-1) \epsilon_t(k_t) \left <\Delta \mathbf{W}_{t}, \mathbf{G}^\text{old}_{t} \right >,
\end{align}
and the plasticity loss:
\begin{align}
\mathbb{E}[\mathcal{L}^{P}_{t}(\mathbf{W}_{t})] = -\frac{k_t}{d} \left <\Delta \mathbf{W}_{t}, \mathbf{G}_{t} \right >.
\end{align}
\end{theorem}

\noindent
The proof of this theorem can be found in Appendix~\ref{appendix:thm:upper_bound_forandcur}. 
Therefore, achieving an optimal balance between these two objectives 
requires solving the following optimization problem:
\begin{align}
k_t^* = \mathop{\text{argmin}}_{k} \left( \mathbb{E}[\mathcal{L}_{t}^{F}(\mathbf{W}_{t})] + \mathbb{E}[\mathcal{L}_{t}^{P}(\mathbf{W}_{t})] \right).
\end{align}
Noting that \(\Delta \mathbf{W}_{t}\) and \(\mathbf{G}_{t}\) gradually change as training progresses, solving this optimization problem is challenging. To simplify this, we introduce a ratio parameter \(\alpha\):
\begin{align}
\alpha = - \frac{\langle \Delta \mathbf{W}_{t}, \mathbf{G}_{t} \rangle}{\langle \Delta \mathbf{W}_{t}, \mathbf{G}^\text{old}_t \rangle}.
\end{align}
This substitution reformulates the optimization problem into the following form:
\begin{align}
\label{eq:k_sim}
k^*_t = \mathop{\text{argmin}}_{k} \left( (t-1) \epsilon_t(k_t) - \alpha \frac{k_t}{d} \right).
\end{align}
Since \(\Delta \mathbf{W}_t\) benefits new tasks, it often interferes with previous task knowledge, leading to:
\begin{align}
\langle \Delta \mathbf{W}_{t}, \mathbf{G}_{t} \rangle > 0, \quad \langle \Delta \mathbf{W}_{t}, \mathbf{\hat{G}}_{t} \rangle < 0.
\end{align}
From Eq.~\eqref{eq:k_sim}, it is evident that increasing \(k_t\) leads to higher \(\epsilon_t(k_t)\), which increases stability loss \(\mathcal{L}_t^F\), while expanding the learning space and thus reducing plasticity loss \(\mathcal{L}_t^P\).
However, since both \(\Delta \mathbf{W}_t\) and \(\mathbf{G}_t\) change during training, \(\alpha\) also varies dynamically. Meanwhile, the update subspace must be determined before training begins for task \(t\). To resolve this mismatch, we treat \(\alpha\) as a fixed hyperparameter throughout the learning process. In our experiments, $\alpha$ was set as a hyperparameter, with $\alpha=20$  as a general choice. Further experimental analysis indicates that our method is highly robust to $\alpha$ in Table~\ref{table:alpha}.

In the simplified optimization problem of Eq.\eqref{eq:k_sim}, the parameter \(k_t\) is restricted to integer values within the range \([1, d]\). The optimal solution to this equation can be obtained by evaluating the objective function for all possible values of \(k_t\) and selecting the one that minimizes it
as $k_t^*$. For example, in ViT-B/16~\citep{dosovitskiy2020image}, embedding
dimension is 768 , and $k$ can be selected as any integer between 1 and 768. It is worth noting that we compute $k$ separately for each LoRA module, as weights at different layers and positions capture substantially different knowledge. A fixed threshold, as used in InfLoRA~\cite{liang2024inflora}, cannot effectively account for such variation across modules.

\begin{algorithm}[tb]
\caption{SplitLoRA}
\label{alg:GAPLoRA}
\begin{algorithmic}[1]  
\State \textbf{Input:} Datasets $\mathcal{D}_t=\{(\mathbf{x}_i^t, y_i^t)\}_{i=1}^{n_t}$, for $T$ tasks $\mathcal{T}_1,...,\mathcal{T}_T$, a pre-trained ViT model $f_{\Theta}(\cdot)$ with $l$ layers.
\State \textbf{Output:} The optimized $\mathbf{W}^l_T$ for each layer $l$.
\State \textbf{Initialization:} $\mathbf{G}^{\text{old}}_1 = \mathbf{0}$
\For{$t = 1$ to $T$}
  \If{$t > 1$}
    \State Compute $k_t$ using Eq.~\eqref{eq:k_sim}  for each LoRA module
    \State Initialize $\mathbf{A}_t$ using Eq.~\eqref{eq:A_proj} for each LoRA module
  \EndIf
  \State Train LoRA on task $\mathcal{T}_t$ using dataset $\mathcal{D}_t$
  \State Update $\mathbf{G}^{\text{old}}_t$ using Eq.~\eqref{eq:old} for each layer
\EndFor
\end{algorithmic}
\end{algorithm}

\subsection{LoRA Updates in the Minor Subspace}
\label{sec:minor-subspace-lora}
To ensure LoRA updates remain within the minor subspace, we fix the projection matrix \( \mathbf{A}_t \) and only optimize \( \mathbf{B}_t \). LoRA parameterizes the weight update as:
\( \Delta \mathbf{W}_t = \mathbf{A}_t \mathbf{B}_t \). When \( \mathbf{A}_t \) is fixed, the update is confined to its column space~\cite{liang2024inflora,wang2023olora}. To restrict this space to the minor subspace of previous tasks, we construct 

\begin{align}
\label{eq:A_proj}
\mathbf{A}_t = \hat{\mathbf{U}}_t^k \mathbf{R}, \end{align}

where \( \hat{\mathbf{U}}_t^k \in \mathbb{R}^{d \times k} \) is an orthonormal basis of the minor subspace and \( \mathbf{R} \in \mathbb{R}^{k \times r} \) is a random Gaussian matrix.
Importantly, this constraint only holds if \( \mathbf{A}_t \) remains fixed during training, otherwise, the update direction may drift out of the subspace. Fortunately, prior works~\cite{zhang2023lorafa,liang2024inflora} verify that fixing \( \mathbf{A}_t \) maintains sufficient model capacity while controlling interference.
This design ensures that task updates are constrained to low-interference directions $\hat{\mathbf{U}}_t^k$, balancing stability and plasticity without additional memory or computational cost. The full procedure of SplitLoRA is summarized in Algorithm~\ref{alg:GAPLoRA}.

\section{Experiment}

\begin{table*}[t]
\setlength{\belowcaptionskip}{-8pt} 
    \caption{We present FAA (\%) and CAA(\%) on ImageNet-R under three incremental learning settings: “5-task,” “10-task,” and “20-task.” All backbone networks are pre-trained on ImageNet-21K.}
    \label{table:results_imagenet}
    \begin{center}
        \resizebox{\textwidth}{!}{
        \setlength\tabcolsep{3.2pt}
        \renewcommand\arraystretch{1.2}
        \begin{tabular}{ l l   c c  c c  c c}
        \toprule
        \multirow{2}{*}{Method} &\multirow{2}{*}{Pub.} 
        &\multicolumn{2}{c}{5-task} &\multicolumn{2}{c}{10-task}  &\multicolumn{2}{c}{20-task} \\
        \cline{3-8}
        &&FAA ($\uparrow$) &CAA ($\uparrow$) &FAA ($\uparrow$) &CAA ($\uparrow$) &FAA ($\uparrow$) &CAA ($\uparrow$) \\
        \hline
        Upper-bound  & --&84.09 $\pm$ 0.21 &-- &84.09 $\pm$ 0.21 &-- &84.09 $\pm$ 0.21 & --\\
        FT   & --&18.74 $\pm$ 0.44 &48.39 $\pm$ 0.58 
                         &10.12 $\pm$ 0.51 &35.23 $\pm$ 0.92 
                         &4.75 $\pm$ 0.40 &22.8 $\pm$ 0.37 \\
        FT++   & --&60.42 $\pm$ 0.87 &71.59 $\pm$ 0.50
                            &48.93 $\pm$ 1.15 &66.79 $\pm$ 0.92 
                            &35.98 $\pm$ 1.38 &59.68 $\pm$ 0.95 \\
        L2P++~\citep{wang2022learning}   &CVPR22 &70.83 $\pm$ 0.58 &78.34 $\pm$ 0.47
                                  &69.29 $\pm$ 0.73 &78.30 $\pm$ 0.69 
                                  &65.89 $\pm$ 1.30 &77.15 $\pm$ 0.65  \\
        Deep L2P++~\citep{wang2022learning}   &CVPR22   &73.93 $\pm$ 0.37 &80.14 $\pm$ 0.54
                                        &71.66 $\pm$ 0.64 &79.63 $\pm$ 0.90 
                                        &68.42 $\pm$ 1.20 &78.68 $\pm$ 1.03 \\
        DualPrompt~\citep{wang2022dualprompt}   &ECCV22   &73.05 $\pm$ 0.50 &79.47 $\pm$ 0.40
                                        &71.32 $\pm$ 0.62 &78.94 $\pm$ 0.72 &67.87 $\pm$ 1.39 &77.42 $\pm$ 0.80 \\
        CODA-P~\citep{smith2023coda}   &CVPR23   &76.51 $\pm$ 0.38 &82.04 $\pm$ 0.54
                                    &75.45 $\pm$ 0.56 &81.59 $\pm$ 0.82 
                                    &72.37 $\pm$ 1.19 &79.88 $\pm$ 1.06 \\
        HiDe-Prompt~\citep{wang2023hierarchical} &NeurIPS23 &76.29 $\pm$ 0.10 &78.77 $\pm$ 0.11 &76.74 $\pm$ 0.18 &78.76 $\pm$ 0.11 &76.46 $\pm$ 0.06 &78.76 $\pm$ 0.11 \\
        EvoPrompt~\citep{kurniawan2024evolving} &AAAI24   &77.16 $\pm$ 0.18 &82.22 $\pm$ 0.54
                               &76.83 $\pm$ 0.08 &82.09 $\pm$ 0.68
                               &74.41 $\pm$ 0.23 &80.96 $\pm$ 1.42 \\
   InfLoRA~\citep{liang2024inflora} &CVPR24   &79.82 $\pm$ 0.27 &84.07 $\pm$ 0.48
    &78.10 $\pm$ 0.43 &83.47 $\pm$ 1.23
    &73.81 $\pm$ 0.47 &81.02 $\pm$ 0.56 \\
        VQ-Prompt~\citep{jiao2024vector}   &NeurIPS24  &79.23 $\pm$ 0.29  &82.96 $\pm$ 0.50
        &78.71 $\pm$ 0.22 & 83.24 $\pm$ 0.68 
        & 78.10 $\pm$ 0.22  & 82.70 $\pm$ 1.16 \\   
        VPT-NSP$^{2}$~\citep{lu2024visual}  &NeurIPS24  
        &79.71 $\pm$ 0.22 & 84.54 $\pm$ 0.68 &79.35 $\pm$ 0.19  &84.92 $\pm$ 0.41
        & 76.72 $\pm$ 0.44  & 82.91 $\pm$ 0.60 \\

        S-LoRA ~\citep{2025slora}&ICLR25                   &79.15 $\pm$ 0.20 &83.01 $\pm$ 0.42
                &77.34 $\pm$ 0.35 &82.04 $\pm$ 0.24
                &75.26 $\pm$ 0.37 &80.22 $\pm$ 0.72 \\
        \textbf{SplitLoRA}  &\textbf{This work}  &\textbf{81.92 $\pm$ 0.29}  & \textbf{85.83 $\pm$ 0.55} &\textbf{81.00 $\pm$ 0.17} &\textbf{85.84 $\pm$ 0.62}  &\textbf{78.82 $\pm$ 0.28}  & \textbf{84.57 $\pm$ 0.44} \\

        \bottomrule
        \end{tabular}
    }
    \end{center}
\end{table*}
\begin{figure*}[t] 
\setlength{\belowcaptionskip}{-10pt} 
    \centering 
    \includegraphics[width=1\textwidth]{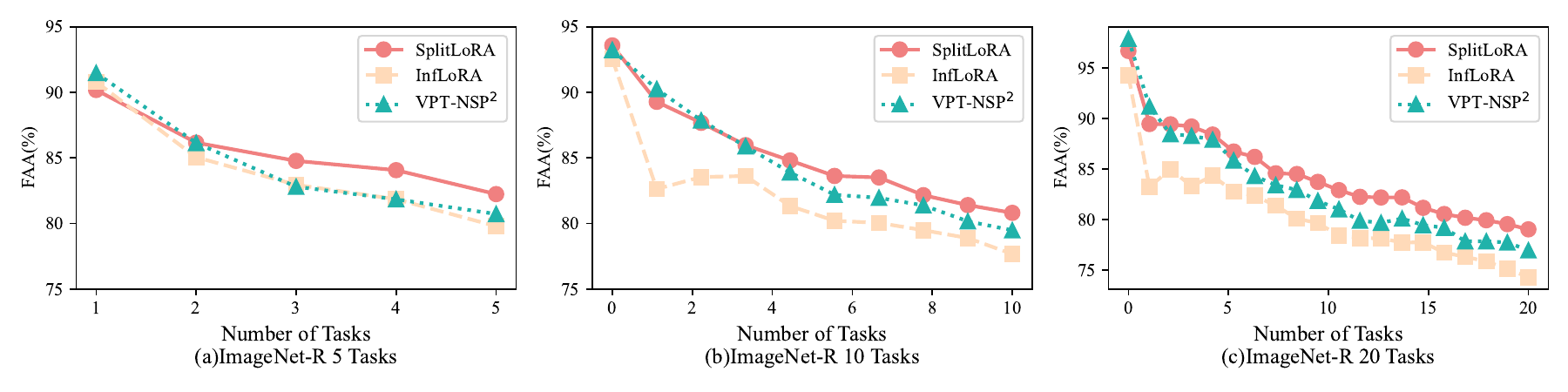} 
    \caption{Variation of the performance of different methods during the learning of ImageNet-R.} 
    \label{fig:imagenet} 
\end{figure*}

\subsection{Experimental Settings}

\textbf{Datasets.} 
We conducted experiments on three standard datasets: ImageNet-R~\cite{hendrycks2021many}, CIFAR-100~\cite{krizhevsky2009learning},  and DomainNet~\cite{peng2019moment}.
ImageNet-R is a variant of ImageNet with 200 classes. CIFAR-100 consists of 100 classes, each containing 600 images. DomainNet contains images from diverse domains, posing a challenge for cross-domain generalization.
Following ~\cite{gao2023unified,wang2022learning,liang2024inflora}, we divided ImageNet-R into 5, 10, and 20 tasks, with each task comprising 40, 20, and 10 classes, respectively. CIFAR-100 was split into 10 tasks, each containing 10 classes, while DomainNet was uniformly partitioned into 5 tasks.
\begin{table*}[t]
\centering
\caption{We present FAA (\%) and CAA(\%) on CIFAR100: 10 tasks and DomainNet: 5 tasks. We report results over 3 trials. All backbone networks are pre-trained on ImageNet-21K.}
    \label{table:results_cifar}
            \resizebox{0.9\linewidth}{!}{
\begin{tabular}{llcc|cc}
\toprule
        \multirow{2}{*}{Method} &\multirow{2}{*}{Pub.}  & \multicolumn{2}{c|}{CIFAR100} & \multicolumn{2}{c}{DomainNet} \\ \cline{3-6} 
 &&FAA ($\uparrow$) &CAA ($\uparrow$) &FAA ($\uparrow$) &CAA ($\uparrow$) \\ \midrule
Upper-bound &--& 91.92 $\pm$ 0.05 & -- & 90.12 $\pm$ 0.13 & -- \\
DualPrompt \cite{wang2022dualprompt} &ECCV22& 84.42 $\pm$ 0.30 & 90.06 $\pm$ 0.07 & 72.14 $\pm$ 0.05 & 77.71 $\pm$ 0.06 \\
CODA-Prompt \cite{smith2023coda} &CVPR23& 86.62 $\pm$ 0.11 & 91.08 $\pm$ 0.28 & 73.23 $\pm$ 0.13 & 78.72 $\pm$ 0.07 \\
LAE \cite{gao2023unified} &ICCV23& 84.15 $\pm$ 0.16 & 89.84 $\pm$ 0.03 & 66.85 $\pm$ 0.40 & 75.01 $\pm$ 0.17 \\
C-LoRA \cite{smith2024continual} &TMLR24& 82.97 $\pm$ 0.47 & 88.81 $\pm$ 0.34 & 69.34 $\pm$ 0.16 & 75.25 $\pm$ 0.11 \\
InfLoRA \cite{liang2024inflora} &CVPR24& 87.06 $\pm$ 0.25 & 91.59 $\pm$ 1.43 & 78.26 $\pm$ 0.50 & 78.82 $\pm$ 0.34 \\
VPT-NSP$^{2}$~\citep{lu2024visual} &NeurIPS24& 88.04 $\pm$ 0.11 & 92.25 $\pm$ 0.80 & 83.83 $\pm$ 0.19 & 88.63 $\pm$ 0.10 \\
\textbf{SplitLoRA} &\textbf{This work}& \textbf{90.33 $\pm$ 0.73} & \textbf{93.70 $\pm$ 0.32} & \textbf{84.31 $\pm$ 0.23} & \textbf{88.99 $\pm$ 0.57} \\
\bottomrule
\end{tabular}
}
\end{table*}

\begin{table*}[t]
\centering
\caption{We present FAA (\%) and CAA(\%) on ImageNet-R: 10-tasks. Backbones are with different self-supervised pre-training paradigms: iBOT-1K and DINO-1K.}
    \label{table:results_pre}
                \resizebox{0.9\linewidth}{!}{
\begin{tabular}{llcc|cc}
\toprule
\multirow{2}{*}{Method} &\multirow{2}{*}{Pub.}  & \multicolumn{2}{c|}{iBOT-1K} & \multicolumn{2}{c}{DINO-1K} \\ \cline{3-6} 
 & & FAA $\uparrow$ & CAA $\uparrow$ & FAA $\uparrow$ & CAA $\uparrow$ \\ \midrule
Upper-bound & -- & 84.09 $\pm$ 0.21 & -- & 81.98 $\pm$ 0.07 & -- \\
DualPrompt \cite{wang2022dualprompt}& ECCV22 & 61.51 $\pm$ 1.05 & 67.11 $\pm$ 0.08 & 58.57 $\pm$ 0.45 & 64.89 $\pm$ 0.15 \\
CODA-Prompt \cite{smith2023coda}& CVPR23 & 66.56 $\pm$ 0.68 & 73.14 $\pm$ 0.57 & 63.15 $\pm$ 0.39 & 69.73 $\pm$ 0.25 \\
HiDe-Prompt \cite{wang2023hierarchical}& NeurIPS23 & 71.33 $\pm$ 0.21 & 73.62 $\pm$ 0.13 & 68.11 $\pm$ 0.18 & 71.70 $\pm$ 0.01 \\
InfLoRA \cite{liang2024inflora}& CVPR24 & 71.84 $\pm$ 0.09 & 78.29 $\pm$ 0.09 & 68.31 $\pm$ 0.28 & 76.15 $\pm$ 0.05 \\
VPT-NSP$^{2}$ \cite{lu2024visual} &NeurIPS24 & 73.85 $\pm$ 0.23 & 80.34 $\pm$ 0.60 & 69.45 $\pm$ 0.74 & 76.38 $\pm$ 0.50 \\ 
VQ-Prompt  \cite{jiao2024vector}&NeurIPS24 & 71.68 $\pm$ 0.72 & 76.66 $\pm$ 0.40 & 68.42 $\pm$ 0.28 & 74.43 $\pm$ 0.58 \\ 
\textbf{SplitLoRA} & \textbf{This work} & \textbf{74.58 $\pm$ 1.05} & \textbf{81.45 $\pm$ 1.72} & \textbf{70.49 $\pm$ 0.31} & \textbf{78.15 $\pm$ 1.13}\\
\bottomrule
\end{tabular}
}
\end{table*}

\textbf{Baselines and Evaluation Metrics.}  
We compare our method with several state-of-the-art continual learning approaches, including L2P++~\citep{wang2022learning}, Deep L2P++~\citep{wang2022learning}, DualPrompt~\citep{wang2022dualprompt}, CODA-P~\citep{smith2023coda}, HiDe-Prompt~\citep{wang2023hierarchical}, EvoPrompt~\citep{kurniawan2024evolving}, VQ-Prompt~\citep{jiao2024vector}, VPT-NSP$^{2}$, and InfLoRA~\citep{liang2024inflora}.  
The “Upper bound” represents the performance achieved by jointly training on all classes in one go. Results are averaged over three runs with different random seeds. 
Following ~\cite{jiao2024vector}, we report Final Average Accuracy (FAA) and Cumulative Average Accuracy (CAA). For  details, please refer to the appendix ~\ref{sec:eval_metrics}.

\noindent\textbf{Implementation Details.} 
We follow prior works~\citep{wang2022learning,wang2022dualprompt,smith2023coda,wang2023hierarchical,zhang2023slca,kurniawan2024evolving} and adopt ViT-Base~\citep{dosovitskiy2020image} pre-trained  on ImageNet-21K~\citep{ridnik2021imagenet} as the backbone. The LoRA rank is set to $10$, and the embedding dimension is $D\!=\!768$, matching the feature dimension of ViT-Base\cite{dosovitskiy2020image}. Following~\cite{liang2024inflora}, we insert SplitLoRA modules into the key and value projections in multi-head attention. Our method is optimized using AdamW~\citep{loshchilov2018decoupled} with an initial learning rate of $1\mathrm{e}{-3}$ for LoRA and $1\mathrm{e}{-2}$ for the classification head. We use a batch size of $256$ across all datasets, and each task is trained for 10 epochs. All experiments are conducted on a single NVIDIA GeForce L40S GPU. All results are reported as mean ± standard deviation over three random seeds.

\begin{table}[t]
\setlength{\belowcaptionskip}{-3pt} 
    \caption{Evaluation of model performance under different values of $\alpha$ on ImageNet-R. A higher $\alpha$ may improve plasticity but could impact stability.}
    \label{table:alpha}
    \begin{center}
        \resizebox{0.7\linewidth}{!}{
        \setlength\tabcolsep{3.2pt}
        \renewcommand\arraystretch{1.2}
        \begin{tabular}{ l    c c  c c  c c}
        \toprule
        \multirow{2}{*}{Method} 
        &\multicolumn{2}{c}{5-task} &\multicolumn{2}{c}{10-task}  &\multicolumn{2}{c}{20-task} \\
        \cline{2-7}
        &FAA ($\uparrow$) &CAA ($\uparrow$) &FAA ($\uparrow$) &CAA ($\uparrow$) &FAA ($\uparrow$) &CAA ($\uparrow$) \\
        \midrule
    InfLoRA   &79.82  &84.07
    &78.10  &83.47
    &73.81  &81.02\\\hline
        SplitLoRA($\alpha=30$)    &{82.15}  & {85.60}
        &{81.03} &{85.56 } 
        &{78.73}  &{84.06} \\
        SplitLoRA($\alpha=20$)   &{81.92}  & {85.83}
        &{81.00} &{85.84 } 
        &{78.82}  & {84.57} \\
        SplitLoRA($\alpha=10$)   &{82.35}  & {85.82}
        &{81.03} &{85.67 } 
        &{77.89}  & {83.27} \\        
        SplitLoRA($\alpha=5$)  &{82.52}  &{85.89}
        &{81.38} &85.89 
        &{78.15}  & {84.19} \\  
        SplitLoRA($\alpha=1$)   &{82.40}  & {85.86}
&{80.89}  &{85.22}
        &{78.59}  & {84.20} \\  
        \bottomrule
        \end{tabular}}
    \end{center}
    
\end{table}

\begin{figure*}[ht] 
\setlength{\belowcaptionskip}{-8pt} 
    \centering 
    \includegraphics[width=1\textwidth]{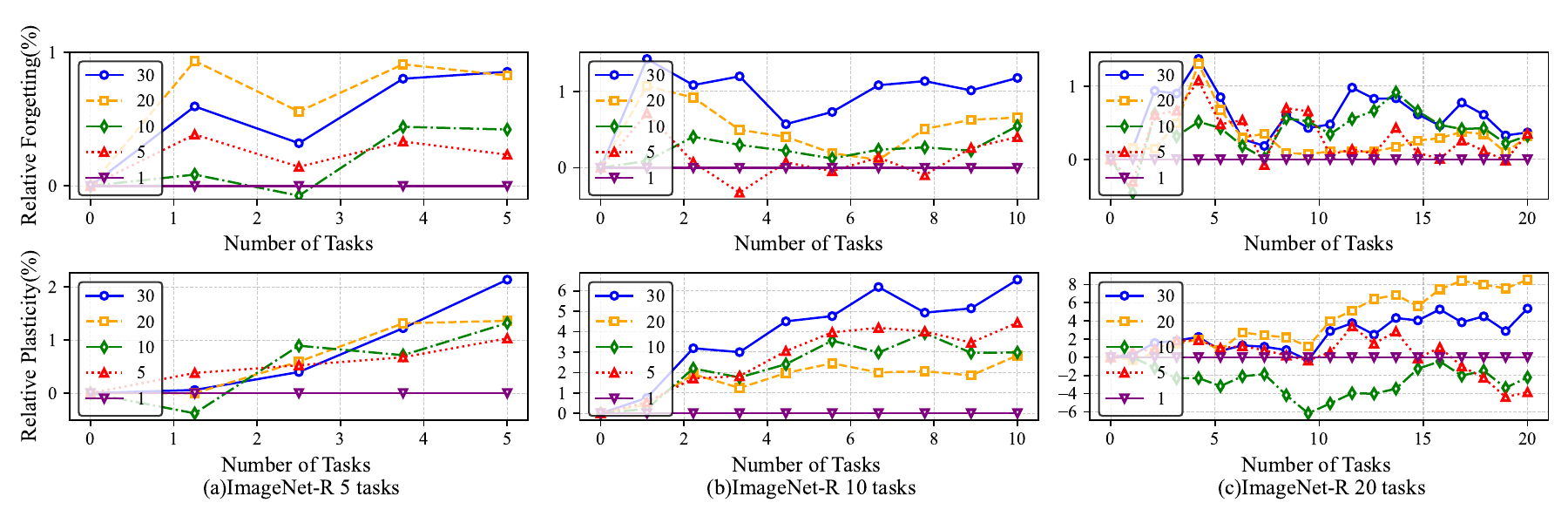} 
    \vspace{-15pt}
    \caption{The impact of $\alpha$ on the stability and plasticity of the model in continual learning. As $\alpha$ increases, stability decreases (higher forgetting) while plasticity improves, illustrating the trade-off between retaining past knowledge and adapting to new tasks.} 
    \label{fig:alpha} 
\end{figure*}

\subsection{Experimental Results}
\textbf{Results on ImageNet-R, CIFAR100 and Domainet.} 
Table~\ref{table:results_imagenet} presents the results of different methods evaluated on ImageNet-R with varying numbers of tasks. 
It highlights how our proposed method, SplitLoRA, achieves consistently higher accuracy compared to existing continual learning methods across different task setups. 
Additionally, Table~\ref{table:results_cifar} shows the results of these methods on CIFAR100 and DomainNet datasets. 
Across both tables, SplitLoRA outperforms other methods in FAA and CAA. Figure~\ref{fig:imagenet} shows the accuracy trends of various CL methods on ImageNet-R. Our method achieves the highest accuracy at the end and outperforms others throughout the learning curve. 

\textbf{Variant Pre-trained Models.}  
Table~\ref{table:results_pre} provides a summary of experimental results on the 10-task ImageNet-R dataset using different self-supervised pre-training paradigms. 
Specifically, we evaluate our method with iBOT-1K~\citep{zhou2022image} and DINO-1K~\citep{caron2021emerging} pre-training frameworks. 
These results clearly demonstrate that SplitLoRA consistently outperforms state-of-the-art continual learning methods, irrespective of the pre-training paradigm used. 
This robustness underscores the generalizability and effectiveness of SplitLoRA in leveraging self-supervised pre-training for continual learning tasks.

\begin{table}[H]
\centering
\begin{minipage}{0.46\linewidth}
\centering
\caption{Impact of different $\mathbf{A}_t$ initialization strategies on ImageNet-R.}
\label{tab:imagenet-r-splits}
\resizebox{\linewidth}{!}{%
\begin{tabular}{lccc}
\toprule
Init of $\mathbf{A}_t$ & 5 tasks & 10 tasks & 20 tasks \\
\midrule
Random & 76.57 & 76.13 & 72.30 \\
InfLoRA & 78.92 & 78.10 & 73.81 \\
SplitLoRA & \textbf{81.92} & \textbf{81.00} & \textbf{78.82} \\
\bottomrule
\end{tabular}}
\end{minipage}
\hfill
\begin{minipage}{0.51\linewidth}
\centering
\caption{Efficiency of LoRA variants on ImageNet-R (10 tasks).}
\label{tab:efficiency}
\resizebox{\linewidth}{!}{%
\begin{tabular}{lccc}
\toprule
Method & Extra Fwd & Mem & Time \\
\midrule
LoRA & None & 22.80 GB & 1h 37m \\
InfLoRA & 2/task & 23.06 GB & 1h 48m \\
SplitLoRA & 1/task & 23.03 GB & 1h 43m \\
\bottomrule
\end{tabular}}
\end{minipage}
\end{table}

\textbf{Initialization strategies of $\mathbf{A}_t$.} Table~\ref{tab:imagenet-r-splits} compares different initialization strategies for $\mathbf{A}_t$. SplitLoRA achieves consistently better performance across task splits, demonstrating the effectiveness of using projected minor subspace over random or InfLoRA.

\textbf{Memory and Time Cost.} Table~\ref{tab:efficiency} shows that SplitLoRA achieves a favorable trade-off between performance and efficiency. It introduces only 1 extra forward pass per task while maintaining similar memory and runtime overheads compared to InfLoRA.

\subsection{Hyperparameter Analysis and Discussion}\label{sec:ablation}
We study the effect of the hyperparameter \(\alpha\) on continual learning performance. As shown in Table~\ref{table:alpha}, changing \(\alpha\) has limited impact on final accuracy, and all settings consistently outperform InfLoRA.
Model stability is measured by \textit{forgetting}, defined as the average gap between each task’s best historical accuracy and its current accuracy. Lower forgetting indicates better knowledge retention. For clarity, we define \textit{relative forgetting} as the difference from the setting where \(\alpha = 1\).
Plasticity is evaluated by the model’s accuracy on the current task. Similarly, \textit{relative plasticity} is defined as the difference from the plasticity when \(\alpha = 1\).

Figure~\ref{fig:alpha} presents results on 5-, 10-, and 20-task splits of ImageNet-R. As \(\alpha\) increases, forgetting grows (lower stability) while plasticity improves. These results show that \(\alpha\) effectively controls the trade-off between retaining past knowledge and adapting to new tasks, while consistently maintaining better performance than InfLoRA across all settings.

\section{Conclusion}
In this paper, we investigate the problem of continual learning based on pre-trained ViT models and propose the SplitLoRA method. Specifically, we partition the gradient space of previous tasks into a major subspace and a minor subspace and theoretically model the impact of the minor subspace size on stability and plasticity. After simplifying the optimization problem, we compute the optimal minor subspace size during the continual learning process. Finally, we employ random projection to map the minor subspace onto the low-dimensional matrix of LoRA. Experiments on multiple benchmark datasets demonstrate that our method effectively achieves state-of-the-art performance.

\textbf{Limitation.} In estimating the optimal subspace size, we assume that the ratio between the gradients of the new and previous tasks remains constant. While experimental results suggest that this assumption is robust and effective in practice, it may not be the most principled or optimal solution.
\clearpage
\bibliographystyle{unsrt} 
\bibliography{cite}
\clearpage

\appendix

\section{Appendix}

\subsection{Proof of Proposition~\ref{thm:upper_bound}}\label{appendix:thm:upper_bound}
Before proving Proposition~\ref{thm:upper_bound}, we first establish a supporting lemma.

\begin{lemma}[Gradient Preservation under Orthogonal Updates]
\label{lemma:gradient_preservation}
Let \( L_j: \mathbb{R}^d \to \mathbb{R} \) be a twice-differentiable loss function corresponding to task \( j \), and let \( W_j \) be the model parameters after completing task \( j \). Suppose at step \( t > j \), the update direction \( \tilde{g}_t \) for task \( t \) satisfies
$\langle \nabla L_j(W_j), \tilde{g}_t \rangle = 0.$
The updated parameter is given by:
$W_t = W_j - \eta \tilde{g}_t.$
Further assume that the second-order term \( \eta H_j \tilde{g}_t \) in the Taylor expansion of \( \nabla L_j \) can be ignored. Then, the gradient of task \( j \) remains unchanged:
\[
\nabla L_j(W_t) = \nabla L_j(W_j).
\]
\end{lemma}

\begin{proof}
Since \( L_j \) is twice-differentiable, we apply the first-order Taylor expansion of the gradient at point \( W_j \) in the direction of \( \tilde{g}_t \):
\[
\nabla L_j(W_t) = \nabla L_j(W_j - \eta \tilde{g}_t) = \nabla L_j(W_j) - \eta H_j \tilde{g}_t + o(\eta).
\]
Now, under the assumption that \( \eta H_j \tilde{g}_t \) is negligible (i.e., small learning rate and low curvature), we ignore the second-order term:

\[
\nabla L_j(W_t) = \nabla L_j(W_j).
\]
\end{proof}

Based on the Lemma ~\ref{lemma:gradient_preservation}, we denote $\nabla L_{i}(W_{j})$ as $\mathbf{G}_{i}$. 
Next, we provide the proof of Proposition~\ref{thm:upper_bound}.

\textbf{Proposition 4.1.} Assume the loss $\mathcal{L}_i(\mathbf{W})$ is $L$-smooth for all $i \in \{1, \dots, t\}$. Let the model update be $\mathbf{W}_t = \mathbf{W}_{t-1} + \Delta \mathbf{W}_t$. Then:
\[
\sum_{i=1}^{t} \left(\mathcal{L}_{i}(\mathbf{W}_{t}) - \mathcal{L}_{i}(\mathbf{W}_{t-1})\right)
\leq  -(t{-}1) \left \langle \Delta \mathbf{W}_t, \mathbf{G}^{\text{old}}_t \right \rangle - \left \langle \Delta \mathbf{W}_t, \mathbf{G}_t \right \rangle +  \frac{(t{-}1) L}{2}\|\Delta \mathbf{W}_t\|^2_F.
\]

\begin{proof}
By $L$-smoothness of each $\mathcal{L}_i$, we have:
\[
\mathcal{L}_i(\mathbf{W}_t) \leq \mathcal{L}_i(\mathbf{W}_{t-1}) + \langle \nabla \mathcal{L}_i(\mathbf{W}_{t-1}), \Delta \mathbf{W}_t \rangle + \frac{L}{2} \|\Delta \mathbf{W}_t\|^2_F.
\]
Summing over $i = 1$ to $t$:
\[
\sum_{i=1}^{t} \mathcal{L}_i(\mathbf{W}_t) - \sum_{i=1}^{t} \mathcal{L}_i(\mathbf{W}_{t-1}) 
\leq \sum_{i=1}^{t} \langle \nabla \mathcal{L}_i(\mathbf{W}_{t-1}), \Delta \mathbf{W}_t \rangle + \frac{tL}{2} \|\Delta \mathbf{W}_t\|^2_F.
\]
Let $\mathbf{G}_t^{\text{old}} = \frac{1}{t-1}\sum_{i=1}^{t-1} \nabla \mathcal{L}_i(\mathbf{W}_{t-1})$, then:
\[
\sum_{i=1}^{t-1} \langle \nabla \mathcal{L}_i(\mathbf{W}_{t-1}), \Delta \mathbf{W}_t \rangle = (t-1) \langle \mathbf{G}_t^{\text{old}}, \Delta \mathbf{W}_t \rangle.
\]
Substituting back gives:
\[
\sum_{i=1}^{t} \left(\mathcal{L}_{i}(\mathbf{W}_{t}) - \mathcal{L}_{i}(\mathbf{W}_{t-1})\right) 
\leq  -(t{-}1) \langle \Delta \mathbf{W}_t, \mathbf{G}_t^{\text{old}} \rangle - \langle \Delta \mathbf{W}_t, \mathbf{G}_t \rangle + \frac{(t{-}1)L}{2} \|\Delta \mathbf{W}_t\|^2_F.
\]
\end{proof}

\subsection{Proof of Theorem~\ref{thm:upper_bound_forandcur}}\label{appendix:thm:upper_bound_forandcur}

Let \( \mathbf{W}_{t-1} \) denote the weight matrix of a linear layer in the model, updated as  $\mathbf{W}_t = \mathbf{W}_{t-1} + \Delta \mathbf{\hat W}_{t}= \mathbf{W}_{t-1} +  \mathbf{U^k}_t \mathbf{U^k}_t^{\top} \Delta \mathbf{W}_{t}$ . Since the update direction of the new task is unknown, we assume that it is uniformly distributed across all directions. that is to say, 
 \( \Delta \mathbf{W}_t \) has the same expected projection value across different feature directions of \( \mathbf{G}_t \), we provide the expected values of the stability loss :
\begin{align}
\mathbb{E}[\mathcal{L}_{t}^{S}(\mathbf{W}_{t})] = -(t-1) \epsilon_t(k_t) \left <\Delta \mathbf{W}_{t}, \mathbf{G}^\text{old}_{t} \right >,
\end{align}
and the plasticity loss:
\begin{align}
\mathbb{E}[\mathcal{L}^{P}_{t}(\mathbf{W}_{t})] = -\frac{k_t}{d} \left <\Delta \mathbf{W}_{t}, \mathbf{G}_{t} \right >.
\end{align}

\begin{proof}
Stability Loss:

By definition, the projected update is:
\[
\Delta \hat{\mathbf{W}}_t = \mathbf{U}_t^k \mathbf{U}_t^{k\top} \Delta \mathbf{W}_t.
\]

Thus, the expected stability loss is:
\begin{align*}
\mathbb{E}[\mathcal{L}_t^S] &= -(t{-}1) \mathbb{E}[\langle \Delta \hat{\mathbf{W}}_t, \mathbf{G}_t^{\text{old}} \rangle] \\
&= -(t{-}1) \mathbb{E}[\langle \mathbf{U}_t^k \mathbf{U}_t^{k\top} \Delta \mathbf{W}_t, \mathbf{G}_t^{\text{old}} \rangle] \\
&= -(t{-}1) \mathbb{E}[\text{Tr}(\Delta \mathbf{W}_t^\top \mathbf{U}_t^k \mathbf{U}_t^{k\top} \mathbf{G}_t^{\text{old}})].
\end{align*}

Let \( \mathbf{G}_t^{\text{old}} = \sum_{i=1}^d \sigma_i \mathbf{u}_i \mathbf{v}_i^\top \) be the SVD. Then,
\[
\mathbf{U}_t^k \mathbf{U}_t^{k\top} \mathbf{G}_t^{\text{old}} = \sum_{i=d-k_t+1}^d \sigma_i \mathbf{u}_i \mathbf{v}_i^\top.
\]

So:
\[
\mathbb{E}[\mathcal{L}_t^S] = -(t-1) \sum_{i=d-k_t+1}^{d} \sigma_i \cdot \mathbb{E}[\langle \Delta \mathbf{W}_t, \mathbf{u}_i \mathbf{v}_i^\top \rangle_F].
\]

Under the uniform distribution assumption, all expected projections are equal:
\[
\mathbb{E}[\langle \Delta \mathbf{W}_t, \mathbf{u}_i \mathbf{v}_i^\top \rangle_F] = c, \quad \forall i.
\]

Then:
\[
\mathbb{E}[\mathcal{L}_t^S] = -(t-1) \cdot c \cdot \sum_{i=d-k_t+1}^{d} \sigma_i.
\]

Also,
\[
\langle \Delta \mathbf{W}_t, \mathbf{G}_t^{\text{old}} \rangle = \sum_{i=1}^{d} \sigma_i \cdot \langle \Delta \mathbf{W}_t, \mathbf{u}_i \mathbf{v}_i^\top \rangle_F = c \cdot \sum_{i=1}^d \sigma_i,
\]
so:
\[
c = \frac{\langle \Delta \mathbf{W}_t, \mathbf{G}_t^{\text{old}} \rangle}{\sum_{i=1}^d \sigma_i}.
\]

Thus,
\[
\mathbb{E}[\mathcal{L}_t^S] = -(t-1) \cdot \epsilon_t(k_t) \cdot \langle \Delta \mathbf{W}_t, \mathbf{G}_t^{\text{old}} \rangle.
\]

---

Plasticity Loss:

The plasticity loss is:
\[
\mathbb{E}[\mathcal{L}_t^P] = -\mathbb{E}[\langle \Delta \hat{\mathbf{W}}_t, \mathbf{G}_t \rangle] = -\mathbb{E}[\langle \mathbf{U}_t^k \mathbf{U}_t^{k\top} \Delta \mathbf{W}_t, \mathbf{G}_t \rangle].
\]

Let \( \alpha_i = \langle \Delta \mathbf{W}_t, \mathbf{u}_i \rangle \), \( \beta_i = \langle \mathbf{G}_t, \mathbf{u}_i \rangle \). Then:
\[
\mathbb{E}[\mathcal{L}_t^P] = - \mathbb{E}\left[\sum_{i=1}^{k_t} \alpha_i \beta_i \right].
\]

Under the uniform assumption, the expected contribution over any direction is \( \frac{1}{d} \), hence:
\[
\mathbb{E}[\mathcal{L}_t^P] = - \frac{k_t}{d} \cdot \langle \Delta \mathbf{W}_t, \mathbf{G}_t \rangle.
\]
\end{proof}

\subsection{Evaluation metrics} 
\label{sec:eval_metrics}

To evaluate continual learning performance, we track the average classification accuracy over all classes encountered so far at the end of each task’s training following \cite{jiao2024vector}. We denote by \(A_{ij}\) the average accuracy on the \(i\)-th task after training the \(j\)-th task. Below, we provide formal definitions for two key metrics: FAA and CAA.

\paragraph{(i) Final Average Accuracy (FAA).}
FAA measures the overall performance after learning all tasks, defined as:
\begin{equation}
\text{FAA} = \frac{1}{T} \sum_{i=1}^{T} A_{iT},
\end{equation}
where \(T\) is the total number of tasks and \(A_{iT}\) is the accuracy for task \(i\) after completing task \(T\). A larger FAA indicates a stronger ability to learn while minimizing forgetting. In some literature, FAA is also referred to as “Last-Acc.”

\paragraph{(ii) Cumulative Average Accuracy (CAA).}
CAA is the average of the FAA values computed after each task is learned, given by:
\begin{equation}
\text{CAA} = \frac{1}{T} \sum_{j=1}^{T} \frac{1}{j} \sum_{i=1}^{j} A_{ij}.
\end{equation}
It captures the overall performance at every incremental step. This metric is sometimes referred to as “Inc-Acc.”
\subsection{The size of minor subspace evolves during training.} 
We tracked the evolution of the minor subspace size throughout training on ImageNet-R with 20 tasks. As shown in the figure ~\ref{fig:ik}, as the number of tasks increases, model stability becomes more critical, leading to a progressively smaller minor subspace. Furthermore, when comparing different layers of ViT, the minor subspace is larger in shallower layers and gradually decreases as the layer depth increases. This suggests that changes in the deep-layer parameters have a greater impact on model stability.

When the value of $\alpha$ varies, the model's learning space changes significantly; nevertheless, the model is still able to learn tasks effectively in all cases. Here, we introduce a simple experiment to explain it.
\begin{table}[h]
\centering
\caption{Average accuracy of tasks 2–20 under different fine-tuning strategies.}
\label{tab:partial_finetune}
\begin{tabular}{l c}
\toprule
{Method} & {Avg. Acc (Tasks 2--20)} \\
\midrule
Only head & 66.08 \\
Only head and the first task & 74.75 \\
SplitLoRA &  81.47\\
\bottomrule
\end{tabular}
\end{table}
As Tab. ~\ref{tab:partial_finetune} shows, "Only head" means training only the classifier head for each task; "Only head and the first task" means training the first task's LoRA and the classifiers for all tasks. SplitLoRA is used for comparison. It can be observed that even when only the classifier is trained, the model can still learn effectively, and when subsequent tasks are fine-tuned on the first task's LoRA, performance improves. The knowledge contained in the pre-trained model and the beneficial knowledge from old tasks help the new task's learning. 

\begin{figure*}[ht] 
    \centering 
    \includegraphics[width=1\textwidth]{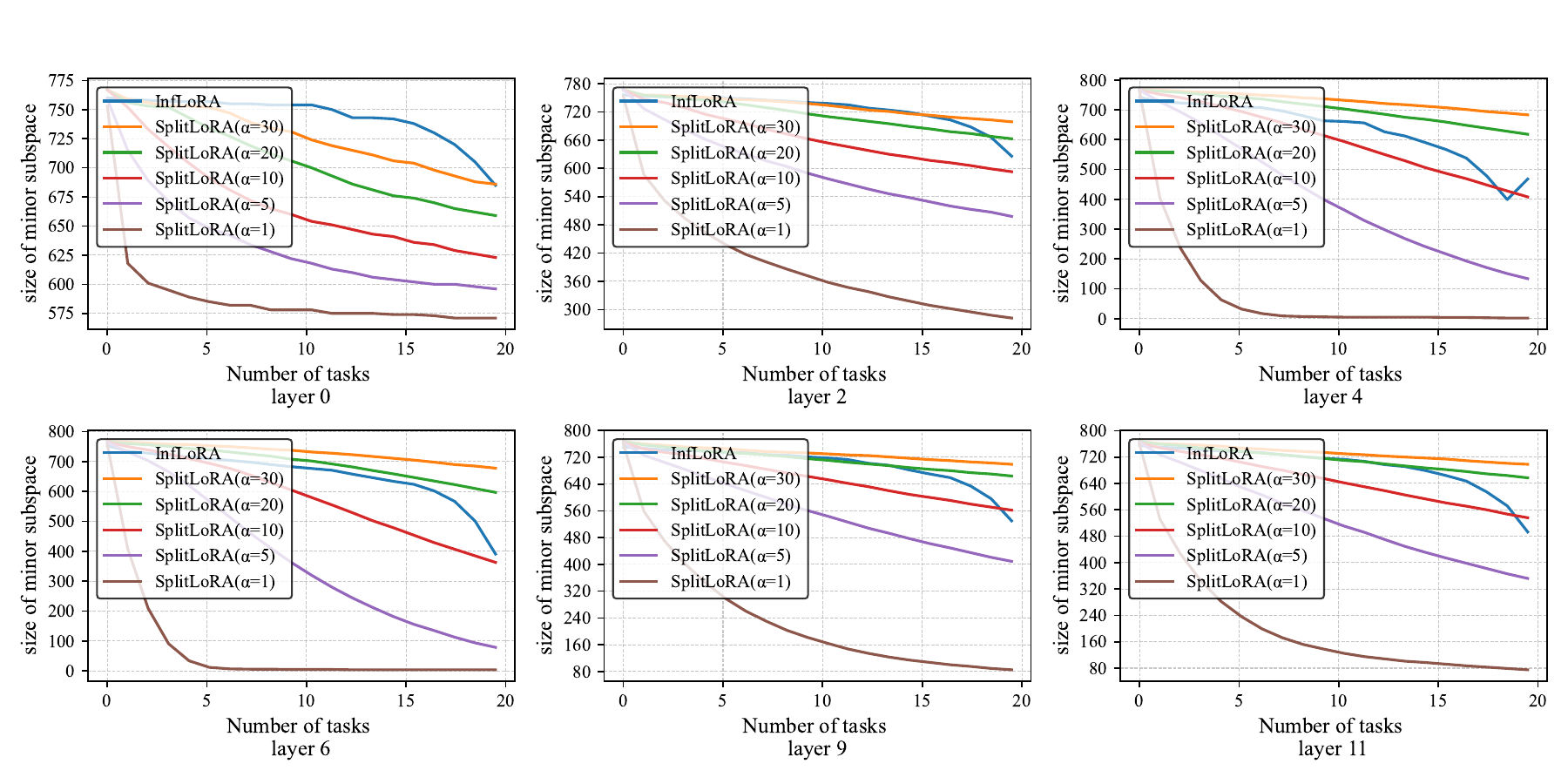} 

    \caption{We recorded the evolution of the minor subspace size during training on ImageNet-R with 20 tasks.} 
    \label{fig:ik} 
    
\end{figure*}

\subsection{More results on other benchmark.} 

\begin{table*}[t]
\centering
\caption{Comparison with state-of-the-art methods on multiple benchmarks. We report CAA and FAA (\%) on base and incremental sessions.}
\label{tab:overall_results}
\resizebox{\textwidth}{!}{
\begin{tabular}{lccccccc|c}
\toprule
\textbf{Method} & \textbf{CIFAR B0 Inc5} & \textbf{CUB B0 Inc10} & \textbf{IN-R B0 Inc5} & \textbf{IN-A B0 Inc20} & \textbf{Obj B0 Inc10} & \textbf{Omni B0 Inc30} & \textbf{VTAB B0 Inc10} & \textbf{Average} \\
\midrule
L2P & 85.94 / 79.93 & 67.05 / 56.25 & 66.53 / 59.22 & 49.39 / 41.71 & 63.78 / 52.19 & 73.36 / 64.69 & 77.11 / 77.10 & 65.30 \\
DualPrompt & 87.87 / 81.15 & 77.47 / 66.54 & 63.31 / 55.22 & 53.71 / 41.67 & 59.27 / 49.33 & 73.92 / 65.52 & 83.36 / 81.23 & 67.11 \\
CODA-Prompt & 89.11 / 81.96 & 84.00 / 73.37 & 64.42 / 55.08 & 53.54 / 42.73 & 66.07 / 53.29 & 77.03 / 68.09 & 83.90 / 83.02 & 69.68 \\
DAP & 94.54 / 90.62 & 94.76 / 94.63 & 80.61 / 74.76 & 54.39 / 46.32 & 72.08 / 59.51 & 86.44 / 80.65 & 84.65 / 84.64 & 78.47 \\
DAP w/o BI & 68.07 / 58.16 & 65.27 / 52.05 & 50.40 / 37.99 & 34.48 / 21.84 & 50.47 / 37.55 & 65.43 / 52.53 & 79.63 / 79.87 & 53.83 \\
SimpleCIL & 87.57 / 81.26 & 92.20 / 86.73 & 62.58 / 54.55 & 59.77 / 48.91 & 65.45 / 53.59 & 79.34 / 73.15 & 85.99 / 84.38 & 72.53 \\
ADAM + VPT-D & 88.46 / 82.17 & 91.02 / 84.99 & 68.79 / 60.48 & 58.48 / 48.52 & 67.83 / 54.65 & 81.05 / 74.47 & 86.59 / 83.06 & 73.61 \\
ADAM + SSF & 87.78 / 81.98 & 91.72 / 86.13 & 68.94 / 60.60 & 61.30 / 50.03 & 69.15 / 56.64 & 80.53 / 74.00 & 85.66 / 81.92 & 74.02 \\
ADAM + Adapter & 90.65 / 85.15 & 92.21 / 86.73 & 72.35 / 64.33 & 60.47 / 49.37 & 67.18 / 55.24 & 80.75 / 74.37 & 85.95 / 84.35 & 74.93 \\
RanPAC & 93.51 / 89.30 & 93.13 / 89.40 & 75.74 / 68.75 & 64.16 / 52.86 & 71.67 / 60.08 & 85.95 / 79.55 & 92.56 / 91.83 & 79.17 \\
EASE & 91.51 / 85.80 & 92.23 / 86.81 & 78.31 / 70.58 & 65.34 / 55.04 & 70.84 / 57.86 & 81.11 / 74.85 & 93.61 / 93.55 & 78.39 \\
HiDe-Prompt & 91.22 / 89.92 & 89.75 / 89.46 & 76.20 / 74.56 & 61.41 / 49.27 & 70.13 / 62.84 & 76.60 / 77.01 & 91.24 / 92.78 & 78.02 \\
ESN & 87.15 / 80.37 & 65.69 / 63.10 & 60.69 / 55.13 & 44.06 / 31.07 & 63.73 / 52.55 & 75.32 / 66.57 & 81.52 / 62.15 & 63.50 \\
\midrule
{SplitLoRA} & {93.11 / 90.84} & {91.52 / 87.46} & {81.74 / 74.81} & {66.11 / 58.22} & {68.60 / 61.33} & {82.30 / 76.86} & {94.39 / 91.95} & \textbf{79.95} \\
\bottomrule
\end{tabular}
}
\end{table*}
We further evaluate the performance of SplitLoRA on another benchmark~\cite{zhou2024continual}.
Tab. ~\ref{tab:overall_results} compares SplitLoRA with state-of-the-art continual learning methods on seven benchmarks. SplitLoRA achieves the best average performance (79.95\%) and consistently ranks top in individual tasks. In particular, it excels on ImageNet-R, Omni, and VTAB, demonstrating strong generalization and knowledge retention. This confirms that SplitLoRA effectively balances stability and plasticity across diverse scenarios.

\end{document}